\documentclass{article}
\usepackage{ijcai24}

\usepackage{times}
\usepackage{soul}
\usepackage{url}
\usepackage[hidelinks]{hyperref}
\usepackage[utf8]{inputenc}
\usepackage[small]{caption}
\usepackage{graphicx}
\usepackage{amsmath}
\usepackage{amsthm}
\usepackage{booktabs}
\usepackage{algorithm}
\usepackage{algorithmic}
\usepackage[switch]{lineno}
\usepackage{multirow}
\usepackage{subfigure}
\usepackage{makecell}
\usepackage{amssymb}
\usepackage{float}
\usepackage{xcolor}
\usepackage{color}

\newtheorem{theorem}{Theorem}

\title{FedPFT: Federated Proxy Fine-Tuning of Foundation Models}

\author{
Zhaopeng Peng$^1$\and
Xiaoliang Fan$^{1,}$\footnote{Corresponding Author}\and
Yufan Chen$^{1}$\and
Zheng Wang$^1$\and
Shirui Pan$^2$\and
Chenglu Wen$^1$\and
Ruisheng Zhang$^3$\and
Cheng Wang$^1$
\\
\affiliations
$^1$Fujian Key Laboratory of Sensing and Computing for Smart Cities, School of Informatics, Xiamen University\\
$^2$School of Information and Communication Technology, Griffith University\\
$^3$School of Information Science and Engineering, Lanzhou University\\
\emails
pengzhaopeng@stu.xmu.edu.cn,
fanxiaoliang@xmu.edu.cn,
\{23020231154174, zwang\}@stu.xmu.edu.cn,
s.pan@griffith.edu.au,
clwen@xmu.edu.cn,
zhangrs@lzu.edu.cn,
cwang@xmu.edu.cn 
}

\begin{document}
    \maketitle

\begin{abstract}
Adapting Foundation Models (FMs) for downstream tasks through Federated Learning (FL) emerges a promising strategy for protecting data privacy and valuable FMs. Existing methods fine-tune FM by allocating sub-FM to clients in FL, however, leading to suboptimal performance due to insufficient tuning and inevitable error accumulations of gradients. In this paper, we propose Federated Proxy Fine-Tuning (FedPFT), a novel method enhancing FMs adaptation in downstream tasks through FL by two key modules. 
First, the sub-FM construction module employs a layer-wise compression approach, facilitating comprehensive FM fine-tuning across all layers by emphasizing those crucial neurons. 
Second, the sub-FM alignment module conducts a two-step distillations—layer-level and neuron-level—before and during FL fine-tuning respectively, to reduce error of gradient by accurately aligning sub-FM with FM under theoretical guarantees. 
Experimental results on seven commonly used datasets (i.e., four text and three vision) demonstrate the superiority of FedPFT. Ours code is at https://github.com/pzp-dzd/FedPFT.

\end{abstract}

\section{Introduction}

\begin{figure*}
\centering
    \subfigure[Two types of proxy sub-FM construction]{
    \centering
    \includegraphics[scale=0.37]{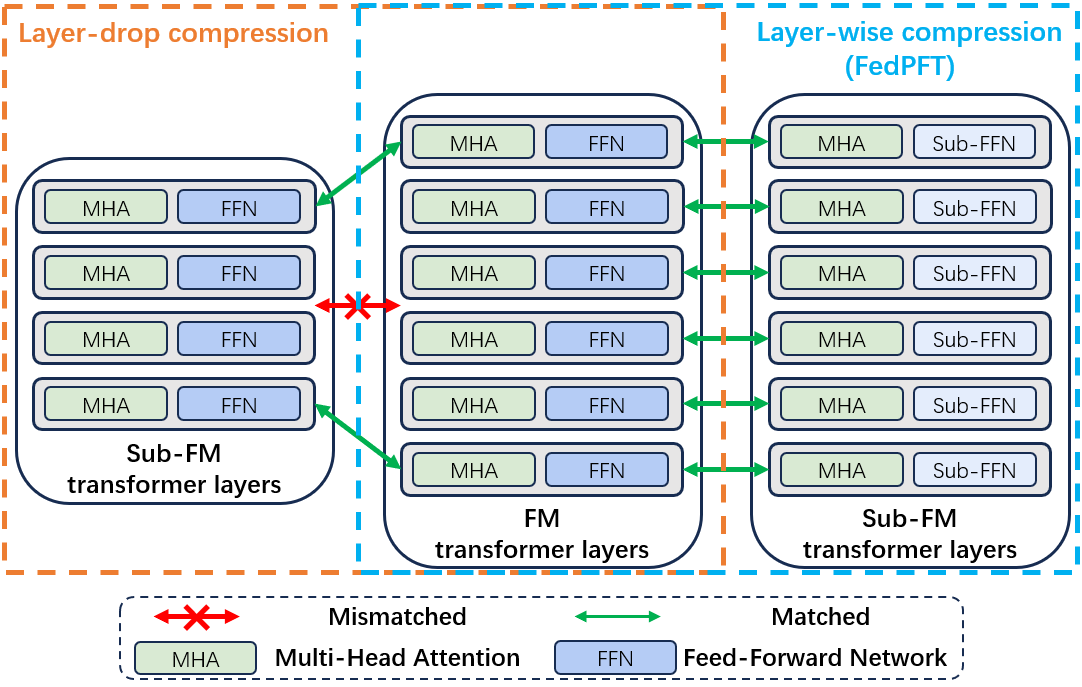}
    \label{fig1_a}
    }
    \subfigure[The problem of accumulating gradients errors]{
    \centering
    \includegraphics[scale=0.253]{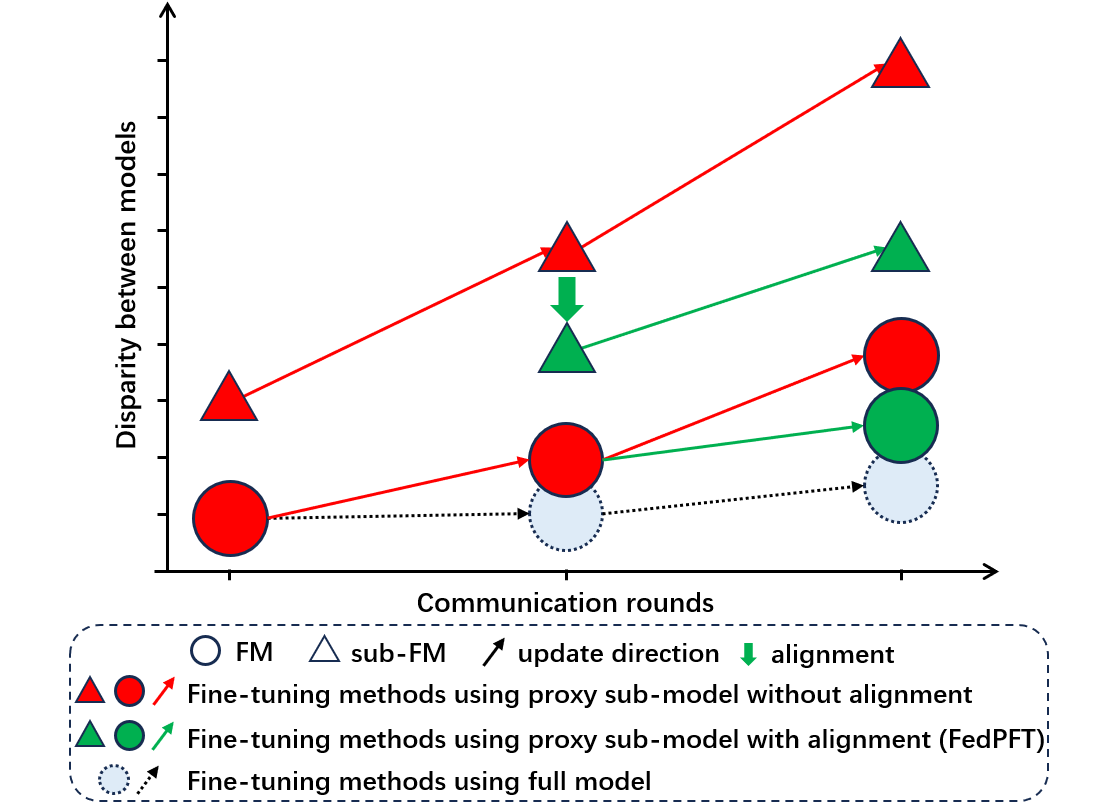}
    \label{fig1_b}
    }
    \caption{A motivating example of two challenges in FM fine-tuning using proxy sub-model. (a) 
    Existing methods constructing sub-FMs via layer-drop compression discard intermediate layers in FM, causing mismatched and insufficient fine-tuning, while FedPFT conducting layer-wise compression ensures comprehensive fine-tuning of FM; and (b) as FL fine-tuning progresses, the discrepancy between the updates made by sub-FMs and FMs grows, leading to a deviation from the ideal update direction, while FedPFT aims to mitigate this gap by accurately aligning sub-FMs and FMs.}
\end{figure*}
In recent years, various transformer-based Foundation Models (FMs) \cite{bommasani2021opportunities} such as BERT \cite{kenton2019bert}, GPT \cite{radford2018improving}, LLaMA \cite{touvron2023llama}, and ViT \cite{dosovitskiy2020image} have attained state-of-the-art performance across a diverse range of natural language processing (NLP) and computer vision (CV) tasks, yet also face both data privacy and FM copyright concerns. 
For instance, a FM trained on medical data might inadvertently memorize sensitive patient information, and companies that own closed-source FMs may choose not to share FMs with the public. Federated Learning (FL) \cite{mcmahan2017communication} offers a privacy-preserving approach for collaborative fine-tuning of FMs among multiple participants. This approach is increasingly promising for FM fine-tuning applications, ensuring the adaptation of downstream tasks without directly sharing client data and server FM.\par
Recent methods \cite{xiao2023offsite,marchisio-etal-2023-mini} mainly aim to fine-tune FMs without using the full model, which leverage layer-drop techniques \cite{sajjad2023effect} to compress a FM and derive a sub-FM, enabling approximate fine-tuning of the original FM. However, these methods still pose \textbf{two significant challenges} that adversely reduce the performance of fine-tuned FMs.
\textbf{On one hand}, they failed to fine-tune FMs sufficiently as a result of discarding those intermediate layers of FMs, consequently leading to the performance degradation of fine-tuned FMs. 
As shown in Fig.\ref{fig1_a}, 
layer-drop methods fail to update intermediate layers of the FM during fine-tuning, due to the mismatch between the FM and the constructed sub-FM.
\textbf{On the other hand}, there is a potential defect for the accumulation of gradient errors of FMs due to the lack of alignment between sub-FMs and FMs during FL fine-tuning, subsequently leading to further performance degradation. Fig.\ref{fig1_b} shows that, due to the absence of alignment, existing methods might accumulate significant gradients update errors between the FM and its constructed sub-FM during the FL fine-tuning process.\par
To address the above two challenges, we propose a framework called \underline {Fed}erated \underline {P}roxy \underline {F}ine-\underline {T}uning (FedPFT) to enhance the adaptation of FMs for downstream tasks, while neither server FMs nor client data are directly shared. 
First, we design the sub-FM construction module, which performs layer compression on FMs to obtain sub-FMs by measuring neurons saliency of Feed-Forward Network (FFN) in transformer, 
facilitating comprehensive fine-tuning of FMs by emphasizing those crucial neurons. 
Second, we design the sub-FM alignment module, which conducts a two-step distillations—layer-level and neuron-level—before and during FL fine-tuning respectively, ensuring the accurate alignment between sub-FMs and FMs with a theoretical guarantee. Extensive experiments on three FMs (i.e., BERT-base, RoBERTa-base, and ViT-base) and seven commonly used datasets (i.e., SST-2, QNLI, MNLI, QQP, CIFAR-10, CIFAR-100, Flowers) demonstrate that FedPFT outperforms existing baselines that fine-tune FMs without using the full model.

\par 

Our contributions can be summarized as follows:
\begin{itemize}

\item 
We introduced FedPFT, a novel federated fine-tuning of FM method that establishes a sub-FM as a local proxy. FedPFT effectively improves fine-tuning performance while maintaining the critical constraint that neither the server FM nor the client data is directly shared.

\item 
We propose the first module for constructing sub-FMs through layer-wise compression. This technique maintains layer correspondence across sub-FMs and FMs, ensuring the comprehensive fine-tuning of FM layers while also considering the alleviation of training overhead.

\item 
We proposed the second module to align sub-FMs with FMs via a two-step distillation—layer-level and neuron-level—before and during FL fine-tuning respectively. Additionally, we offer theoretical insights into the significance of distillation for fine-tuning using sub-model.

\item We conducted extensive experiments on three FMs and seven commonly used datasets. Results demonstrate that FedPFT consistently outperforms existing baselines.
\end{itemize}
\section{Related Works}
\subsection{FM Fine-tuning through FL}
Traditional centralized fine-tuning faces privacy concerns due to data sharing. Recent works \cite{chen2023federated,yu2023federated,zhuang2023foundation} introduce the concepts of Federated Foundation Models, to alleviate privacy concerns. \cite{fan2023fate,kuang2023federatedscope} propose various Fed-LLM platforms to support federated training of LLMs. \cite{xu2023federated} fine-tune FM via FL on mobile devices. \cite{chen2023fedbone} apply FM to federated multi-task learning. \cite{ijcai2023p394} save the communication cost during FL training through block-level parameters dropout. \cite{ijcai2023p483} reduce the communication and computation cost by training different layers of BERT in each round. \cite{zhang-etal-2023-fedpetuning} apply parameter-efficient fine-tuning (PEFT) methods to federated fine-tuning of FMs for privacy defense. 
However, most of aforementioned methods rely on sharing the server FM. This limitation may pose risks of FM copyright leakage and impose a substantial computational burden on clients.
\subsection{FM Fine-tuning without using the full model}
Early PEFT methods, including Lora \cite{hu2021lora}, Adapter-Tuning \cite{houlsby2019parameter}, and Prefix-Tuning \cite{li2021prefix}, focus on efficient fine-tuning of complete FMs by reducing the number of tunable parameters. Despite these efforts, the gradient computation for tunable parameters still needs backpropagation through the entire FM \cite{sung2022lst}.
Recently, Offsite-Tuning \cite{xiao2023offsite} is proposed to achieve fine-tuning without the full model. In this approach, the FM owner sends a light-weight adapter and an emulator constructed through layer-drop and knowledge distillation \cite{hinton2015distilling} to clients. Clients then fine-tune the adapter on their data with support from the emulator. The refined adapter is subsequently returned and incorporated into the full model for the fine-tuning process. Similarly, mini-model adaptation \cite{marchisio-etal-2023-mini} constructs a shallow mini-model from a fraction of the FM's parameters. However, those methods either discard significant amount of intermediate layers in FM or face the problem of gradient error accumulation, resulting in sub-optimal fine-tuning performance. Different from conventional methods, we construct sub-FMs based on layer-wise compression and mitigate gradient error accumulation by a two-step distillations.
\section{FedPFT}
\begin{figure*}
    \centering
    \includegraphics[scale=0.44]{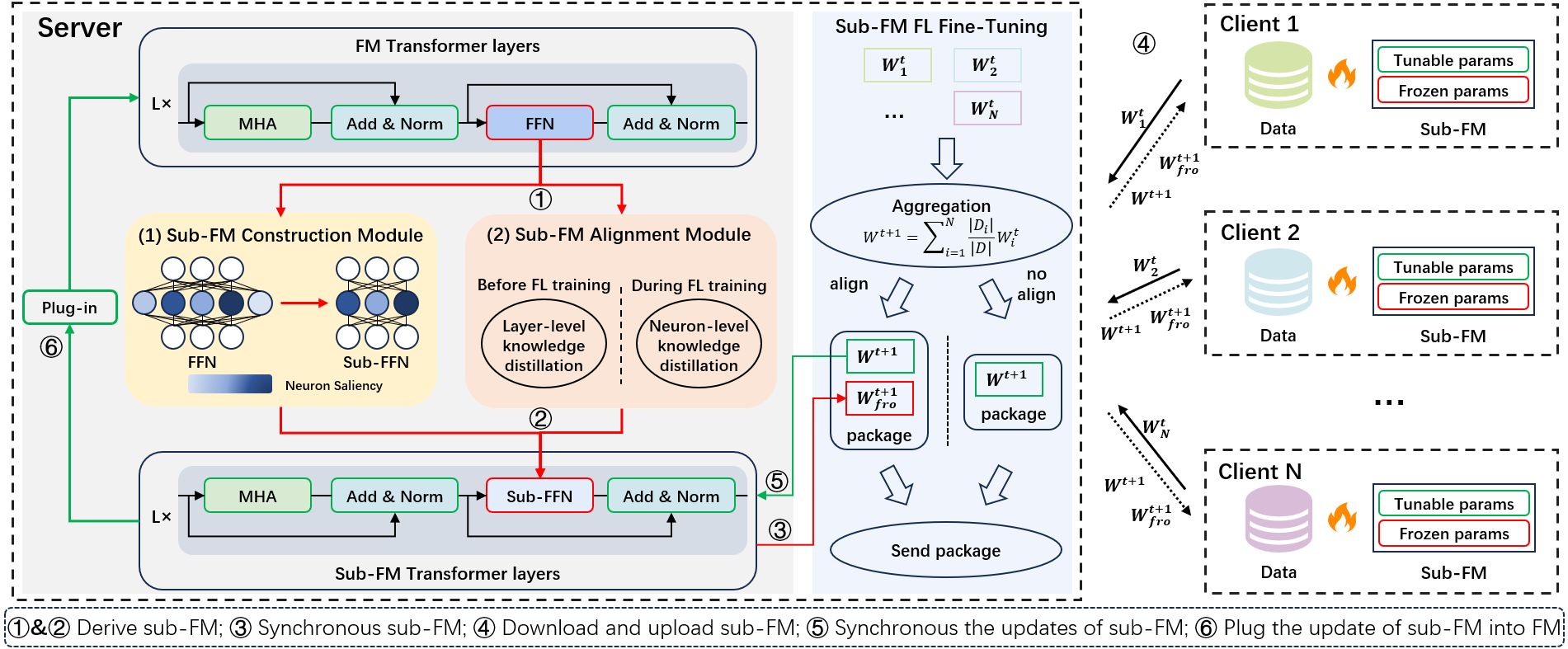}
    \caption{\textbf{The overall framework of FedPFT} that enhances FMs adaptation in downstream tasks through FL by two key modules: (1) Sub-FM Construction Module constructs sub-FM by layer-wise compression to facilitate comprehensive FM fine-tuning; and (2) Sub-FM Alignment Module aligns sub-FM by two-step distillation to ensure accurate alignment between sub-FM and FM with a theoretical guarantee.}
    \label{fig2}
\end{figure*}
\subsection{Preliminary}
\subsubsection{Federated Learning}
Given $N$ parties $P_{i}(i=1, ...., N)$, each party holds data $D_{i}$. Let $L(\cdot,\cdot)$ be the loss function. FL aims to train a machine learning model $\Theta$ using the dataset $D=\cup D_{i}(i=1, ..., N)$ under the coordination of a server $S$, while the raw data of all parties are not directly shared, which is formally described as
\begin{align}
    \Theta = \mathop{\arg\min}\limits_{\Theta}{\sum_{i=1}^{N} \frac{\lvert D_{i} \rvert}{\lvert D \rvert} L(\Theta, D_{i})}.
\end{align}
\subsubsection{Foundation Model Fine-tuning}
Given a foundation model $\Theta=\{W_{1}, W_{2}, ... , W_{n}\}$ and a downstream task dataset $D$, the fine-tuning aims to obtain a new model $\Theta^{*}=\{W_{1}^{*}, W_{2}^{*}, ... , W_{n}^{*}\}$, it is 
\begin{equation}
    \begin{aligned}
    \Theta^{*}&=\Theta+\Delta\Theta,\\
    \Delta\Theta&=\mathop{\arg\min}\limits_{\Delta\Theta}{L(\Theta+\Delta\Theta, D)}.
\end{aligned}
\end{equation}
\subsection{Problem Definition}
\label{problem_def}
\textbf{For FM fine-tuning using proxy sub-model}, we first construct a sub-model $\Theta^{'}=\{W_{1}, W_{2}, ..., W_{k}, W_{k+1}^{'}, ... , W_{n}^{'}\}$ with fewer parameters for $\Theta$ to act as a proxy. Second, fine-tune the proxy sub-model $\Theta^{'}$ using the dataset $D$. Finally synchronize the updated gradients on $\Theta^{'}$ to $\Theta$. Specifically, we construct $\Theta^{'}$ by compressing $\Theta$, and retain a portion of the parameter matrix in $\Theta$ during the compression process. This compression process is formally described as follows:
\begin{align}
    \Theta^{'} = \Theta_{1} \cup C(\Theta_{2}),
\end{align}
where  $\Theta_{1} \cup \Theta_{2} = \Theta$, and $C(\cdot)$ denotes the compression method. During the fine-tuning of $\Theta^{'}$, we update only $\Theta_{1}$ and synchronize the updated gradient on $\Theta_{1}$ into $\Theta$ after fine-tuning to obtain $\Theta^{*}$ 's approximation of $\Theta^{*'}$, which is formally described as 
\begin{equation}
    \begin{aligned}
    \Theta^{*'}&=(\Theta_{1} + \Delta\Theta_{1}^{'}) \cup \Theta_{2},\\
    \Delta\Theta_{1}^{'}&=\mathop{\arg\min}\limits_{\Delta\Theta_{1}^{'}}{L((\Theta_{1} + \Delta\Theta_{1}^{'}) \cup C(\Theta_{2}), D)}.
\end{aligned}
\end{equation}

\subsection{Method Overview}
The overall framework of ours FedPFT is shown in Fig.\ref{fig2}.  We first derive a proxy sub-FM for the server FM, then collaboratively fine-tune the sub-FM through FL, and finally synchronise the updates on the sub-FM to the FM by plugging-in. FedPFT enhances downstream tasks adaptation of FMs through FL by two key module: (1) \textbf{Sub-FM Construction Module} that constructs sub-FMs by performing layer-wise compression on FMs based on neuron saliency; and (2) \textbf{Sub-FM Alignment Module} that reduces the difference between FMs and sub-FMs by layer-level and neuron-level knowledge distillation before and during FL fine-tuning, respectively. We will introduce those two modules in details as follows.
\subsection{Sub-FM Construction Module based on Layer-wise Compression}
\label{3.3}
Transformer-based FM typically consist of three parts: an embedding layer, a task head, and a sequence of transformer layers. Since the size of FM is dominated by all transformer layers, we perform compression for each transformer layer.\par
Each transformer layer contains two sub-layers: Multi-Head Attention (MHA) and Feed-Forward Network (FFN), each of which applies residual connection and followed by layer normalization. The output of MHA is
\begin{equation}
    \begin{aligned}
    {\rm MHA}(x)&={\rm Concat}({\rm Attn}_{0}(x),..., {\rm Attn}_{h}(x))W^{O},\\
    {\rm Attn}(x)&={\rm softmax}(\frac{xW^{Q}({xW^{K}})^{T}}{\sqrt{d_{k}}})xW^{V},
\end{aligned}
\end{equation}
where $W^{Q} \in \mathbb{R}^{d_{model}\times d_{k}}$, $W^{K} \in \mathbb{R}^{d_{model}\times d_{k}}$, $W^{V} \in \mathbb{R}^{d_{model}\times d_{k}}$ and $W^{O} \in \mathbb{R}^{d_{model}\times d_{model}}$ are the weight matrices of query, key, value, and output in MHA, respectively. $h$ is the number of attention heads, $d_{k}$ and $d_{model}$ are the dimensions of key and FM, respectively, and $d_{model} = d_{k} \times h$. The parameters number of MHA is about $4d_{model}^{2}$.\par
The output of FFN is
\begin{align}
    {\rm FFN}(x)&={\rm gelu}(xW_{1}+b_{1})W_{2}+b_{2},
    \label{ffn_output}
\end{align}
where $W_{1} \in \mathbb{R}^{d_{model}\times d_{ff}}$ and $W_{2} \in \mathbb{R}^{d_{ff}\times d_{model}}$ are  the weight matrices of two linear layers in FFN, respectively, $b_{1} \in \mathbb{R}^{d_{ff}}$ and $b_{2} \in \mathbb{R}^{d_{model}}$ are the bias, $d_{ff}$ is the dimensions of FFN and is usually set to $4 \times d_{model}$. The parameters number of FFA is about $8d_{model}^{2}$. Obviously, it is that most of the parameters in transformer layer are contained in FFN.\par

Hence, we opt to compress the FFN rather than the MHA of each layer for sub-FM construction. This minimizes the parameters number of sub-FM while ensuring a consistent set of trainable parameters (i.e. MHA) between the FM and its sub-FM at each layer. We accomplish layer-wise compression by systematically removing neurons with low saliency in the FFN of each layer, employing a fixed ratio.

First, by further transforming (\ref{ffn_output}), we can represent the output of FFN as the sum of $d_{ff}$ neurons outputs:
\begin{align}
{\rm FFN}(x)=\sum_{i=1}^{d_{ff}}{({\rm gelu}(xu_{i}+b_{1 i})w_{i}})+b_{2},
\label{eq9}
\end{align}
where $w_{i} \in \mathbb{R}^{d_{model}}$ is the $i$th column vector in $W_{2}$, $u_{i} \in \mathbb{R}^{d_{model}}$  is the $i$th row vector in $W_{1}$, $b_{1i}$ is the $i$th item in $b_{1}$. \par

Second, based on (\ref{eq9}) and magnitude-based pruning method \cite{wen2016learning}, we use the L2-norm of all connect weights of neuron to measure its saliency, that is:
\begin{align}
    {\rm Saliency}(i)=\sqrt{\sum_{j=1}^{d_{model}}(w_{ij}^{2}+u_{ij}^{2})},
\end{align}
where $i$ is the index of neurons in FFN.\par

Finally, we construct a sub-FM serving as a proxy for the FM, accomplished by systematically eliminating neurons with low saliency in each layer at a fixed ratio.

\subsection{Sub-FM Alignment Module based on Two-step Knowledge Distillation}
In accordance with the description of FM fine-tuning using proxy sub-model in \ref{problem_def}, it is evident that the FM fine-tuning is entirely contingent on the gradient descent of its sub-FM. This fine-tuning methodology prompts a fundamental question: How can we ensure the convergence of FM to the optimal value with the assistance of its sub-FM?
\begin{theorem}
\label{theorem_1}
    Suppose both the function $f: \mathbb{R}^{n}\rightarrow \mathbb{R}$ and its approximation $f': \mathbb{R}^{n}\rightarrow \mathbb{R}$ are convex and differentiable, and their gradient are Lipschitz continuous with constant $L_{1} \textgreater 0$ and $L_{2} \textgreater 0$, respectively, i.e. we have that$ \|\triangledown f(x)-\triangledown f(y)\|_{2} \leq L_{1}\|x-y\|_{2}$ and $ \|\triangledown f'(x)-\triangledown f'(y)\|_{2} \leq L_{2}\|x-y\|_{2}$ for any $x,y$. Then if we run gradient descent for $k$ iterations on $f'$ with a fixed step size $\eta\leq\frac{1}{L_{1}}$ and synchronize the gradient to $f$, let $\triangledown f' - \triangledown f=\delta$, when satisfying 
    \begin{align}
        \|\delta\|_{2}^{2}& \textless \frac{1}{2}\|\triangledown f\|_{2}^{2},
        \label{eq11}
    \end{align}
    \begin{align}
        \eta\sum_{i=1}^{k}\|\delta^{(i)}\|_{2}^{2}&\leq \sum_{i=1}^{k}\langle \delta^{(i)}, x^{(i)}-x^{*} \rangle,
        \label{eq12}
    \end{align}
    it will yield a solution $f^{(k)}$ which satisfies
    \begin{align}
        f(x^{(k)})-f(x^{*}) \leq \frac{\|x^{(0)}-x^{*}\|_{2}^{2}}{2\eta k},
    \end{align}
    where $f(x^{*})$ is the optimal value. 
\end{theorem}
\begin{proof}
    See Appendix.A
\end{proof}
Intuitively, Theorem.\ref{theorem_1} indicates that when (\ref{eq11}) and (\ref{eq12}) are satisfied, gradient descent of FM with the help of sub-FM is guaranteed to converge and converges with rate $O(\frac{1}{k})$. It is evident that both conditions (\ref{eq11}) and (\ref{eq12}) are constraints on the difference between the actual and ideal update gradients of FM, and thus how to minimize the difference of the update gradients becomes a problem to be solved in the next step.
\begin{theorem}
    \label{theorem_2}
    For a transformer, let the number of attention head be 1, and ignoring its nonlinear function and residual connection, its output can be expressed as $y=xW^{Q}(xW_{K})^{T}xW^{V}W^{O}W_{1}W_{2}$, let $A=W^{Q}(W^{K})^{T}$, $B=W^{V}W^{O}$, $C=W_{1}W_{2}$, then $y=xAx^{T}xBC$, and the output of its corresponding sub-layer after compressing FFN layer is expressed as $y'=xAx^{T}xBC'$, assuming that the gradient of loss function $loss=f(y)$ is Lipschitz continuous with constant $L_{3} \textgreater 0$ and $\|C-C'\|_{2}^{2} \leq \epsilon_{1}$, $\|y-y'\|_{2}^{2} \leq \epsilon_{2}$, there exists the constant $K_{1} \textgreater 0$ and $K_{2}\textgreater 0$ such that
    \begin{equation}
        \begin{aligned}
        \|\frac{\partial loss'}{\partial A}-\frac{\partial loss}{\partial A}\|_{2}^{2} &\leq K_{1}\epsilon_{1} + K_{2}\epsilon_{2}, \\
        \|\frac{\partial loss'}{\partial B}-\frac{\partial loss}{\partial B}\|_{2}^{2} &\leq K_{1}\epsilon_{1} + K_{2}\epsilon_{2},
    \end{aligned}
    \end{equation}
\end{theorem}
\begin{proof}
    See Appendix.B
\end{proof} 
From Theorem.\ref{theorem_2}, it is evident that shrinking the error of gradients can be achieved by narrowing the difference in output and weights between the sub-FM and FM.\par
Based on the above analysis, we grasp the importance of narrowing the difference between sub-FM and FM via knowledge distillation to boost the performance of FM that fine-tuned using sub-FM. Therefore, we propose a method to align sub-FM using layer-level and neuron-level distillations in two phases, before and during FL fine-tuning, respectively. These two distillation methods are shown in the Fig.\ref{fig3}.
\begin{figure}
    \centering
    \includegraphics[scale=0.32]{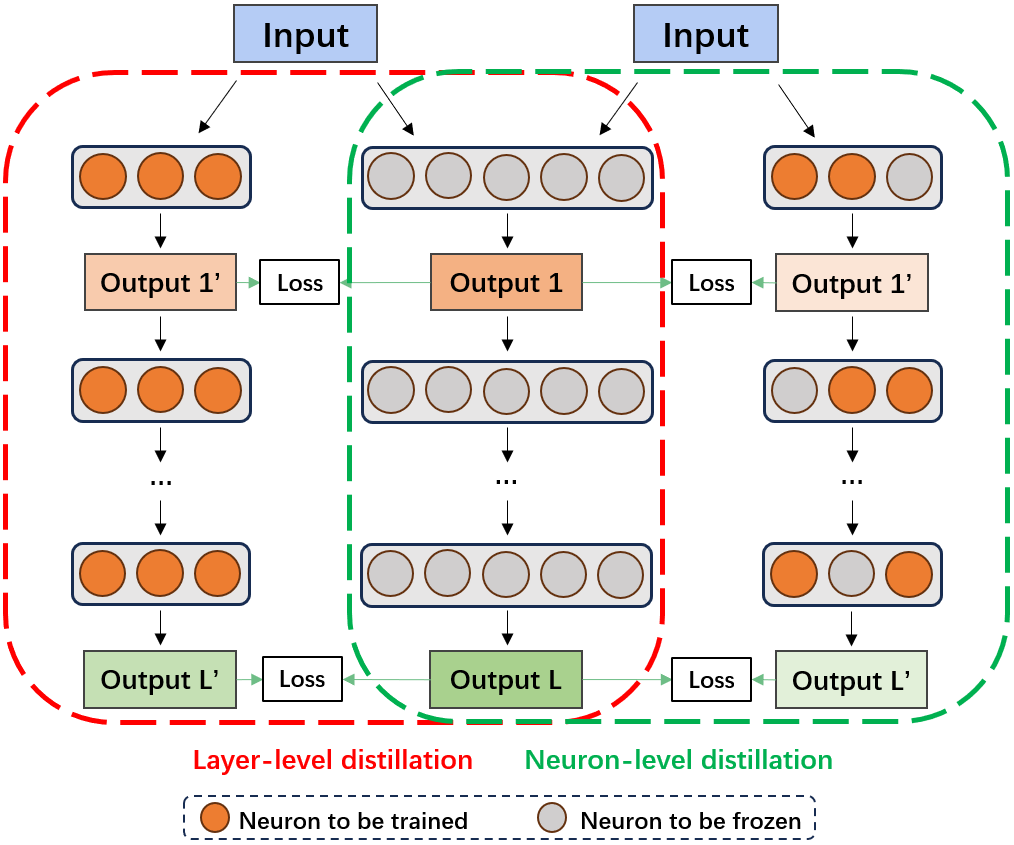}
    \caption{An example of two distillation processes}
    \label{fig3}
\end{figure}
\subsubsection{Layer-level distillation before FL fine-tuning}
Given that our sub-FMs are constructed based layer-wise compression, where each layer retains a set of tunable parameters (i.e., MHA), we leverage the outputs from all layers to compute the layer-level distillation loss.\par 
Furthermore, based on Theorem.\ref{theorem_2}, we enhance the aforementioned distillation loss by introducing a regularization term. The purpose of this regularization term is to quantify the disparity between the weights of FFN and sub-FFN in each layer, to further facilitate a thorough knowledge transfer during fine-tuning by refining the alignment process. Thus, the final distillation loss is denoted as:
\begin{align}
\notag
    L_{KD}=&\frac{1}{LM_{KD}}\sum_{i=1}^{L}((\sum_{j=1}^{M_{KD}}\|O_{j}^{(i)}-O_{j}^{'(i)}\|_{2}^{2})\\
    &+\mu\|W_{1}^{(i)}W_{2}^{(i)}-W_{1}^{'(i)}W_{2}^{'(i)}\|_{2}^{2}),
\end{align}
where $L$ is the number of layers, $M_{KD}$ is the size of distill dataset, $O_{j}^{(i)}$ is the output of the $j$th sample in the $i$th layer, $W_{1}^{(i)}$ and $W_{2}^{(i)}$ are the two weight matrices of the $i$th FFN, $\mu$ is the regularization coefficient.
\subsubsection{Neuron-level distillation during FL fine-tuning}
In addition, the absence of alignment between FM and its constructed sub-FM during FL fine-tuning may cause the actual gradient update direction of the FM to gradually deviate from its ideal direction. This deviation can substantially reduce the performance of the fine-tuned FM. To mitigate the problem, we re-align the sub-FM with the FM after FL aggregation in certain rounds.\par
However, since the datasets for distillation and the datasets for local fine-tuning are typically collected from different domains, excessively aligning sub-FMs through distillation may hinder the adaptation of sub-FMs to downstream tasks. Inspired by \cite{mallya2018packnet}, during the alignment process in FL fine-tuning, we opt to update only a subset of neurons with low saliency in local fine-tuning to prevent the risk of sub-FM forgetting knowledge of local data. \par
Moreover, since all FFNs of sub-FM are not updated during FL fine-tuning, the effectiveness of magnitude-based neuron saliency measurement methods diminishes. To address this, we opt to select neurons for updating during alignment based on the Average Percentage of Zero activations (APoZ) of outputs on the downstream task dataset \cite{hu2016network}. The $APoZ_{k}^{(i)}$ of the $k$th neuron in $i$th layer is defined as:
\begin{align}
    {APoZ}_{k}^{(i)}=\frac{1}{M_{DT}S}\sum_{j=1}^{M_{DT}}\sum_{l=1}^{S}\mathbb{I}(O_{jkl}^{(i)}=0),
\end{align}
where $M_{DT}$ is the size of the downstream task dataset, $S$ is the sequence length of the $j$th sample, $O_{jkl}^{(i)}$ is the output of the $l$th token of the $j$th sample at the $k$th neuron in $i$th layer, $\mathbb{I}(\cdot)$ is the indicator function.\par
We calculate the APoZ for each neuron on the client using the local dataset before each round that requires alignment, and subsequently select the neurons that need to be updated during alignment based on their APoZ values.
\begin{table*}
    \centering
    \begin{tabular}{cccccccc}
        \hline
         Model& Setting & Method & SST-2 & QNLI & MNLI-(m/mm) & QQP & \makecell[c]{Transformers\\Params}\\
        \hline
        \multirow{3}*{BERT}& \makecell[c]{Fine-tuning \\with full model} &FedPETuning & 91.6 & 88.4 & 80.7/81.6 & 87.8 & 81M\\
         \cline{2-8}
         &\multirow{2}*{\makecell[c]{Fine-tuning \\without full model}}& FedOT  &90.4 &83.9 &74.9/74.9 &81.6 & 47M \\
         
         && FedPFT& \textbf{91.6} & \textbf{87.5} & \textbf{78.6/79.0} & \textbf{86.3} & 47M\\
        \hline
        \multirow{3}*{RoBERTa}& \makecell[c]{Fine-tuning \\with full model}& FedPETuning&  93.6 & 90.8 & 86.1/85.5 & 88.3 &  81M\\
        \cline{2-8}
         &\multirow{2}*{\makecell[c]{Fine-tuning \\without full model}}& FedOT &  92.8& 85.3 &80.6/81.4 &84.6 & 47M  \\
         && FedPFT & \textbf{93.1} & \textbf{88.7} & \textbf{83.4/83.2} & \textbf{87.1} & 47M\\ 
         
        \hline
    \end{tabular}
    \caption{\textbf{Overall experimental results on BERT and RoBERTa.} The evaluation metric is accuracy. For MNLI, 'm' and 'mm' denote the matched and mismatched results, respectively. The best result for the same setting is marked in bold.}
    \label{tab1}
\end{table*}
\subsection{Cost Analysis}
We perform a theoretical analysis of the computational and communication cost of FedPFT based on BERT, and other models such as RoBERTa and ViT are similar. Following the settings in \cite{ijcai2023p483}, we assume that all FL clients have the same training settings and exclude external differences such as local dataset size and hardware environment.
\subsubsection{Computational Cost}
Given a BERT model, let $V$ be the vocabulary size, $S$ be the sequence length, $L$ be the number of layers, and $c_{f}$,$c_{b}$ be the number of forward propagation and backward propagation respectively. Based on the analysis in \ref{3.3}, the computational cost of a BERT model is 
$O(d_{model}(V+S) + L(4Sd_{model}^{2}+S^{2}d_{model})
   +2LSd_{model}d_{ff}+LSd_{model})$
where the four terms denote the cost of embedding, MHA, FFN and Add\&Norm, respectively. \par
Based on the above information, the overall time complexity of the full model is computed as follows. First, the time complexity of embedding is $O(d_{model}(V+S))$. Second, due to $d_{ff}=4 d_{model}$, the forward propagation time complexity is about $ O(c_{f}L(12Sd_{model}^{2}+S^{2}d_{model}))$. Identically, the backward propagation time complexity is $O(c_{b}L(12Sd_{model}^{2}+S^{2}d_{model}))$. Finally, the overall time cost is $O(d_{model}(V+S)+L(c_{f}+c_{b})(12Sd_{model}^{2}+S^{2}d_{model}))$ and we have $d_{model}(V+S) \ll L(c_{f}+c_{b})(12Sd_{model}^{2}+S^{2}d_{model})$ and $S \textless d_{model}$.\par
In FedPFT, because $d'_{ff}=d_{model}$ for sub-FM, the time complexity during local training is $O(L(c_{f}+c_{b})(6Sd_{model}^{2}+S^{2}d_{model}))$. Thus, compared with full model, FedPFT could reduce almost half the computational cost of all clients.
\subsubsection{Communication Cost}
Following \cite{vaswani2017attention}, the space complexity of a BERT model is $O(d_{model}(V+S)+12Ld_{model}^{2}+2Ld_{model})$. \par
In FedPFT, if PEFT methods is not used, all model parameters need to be transmitted in each iteration, thus the communication cost will be in the complexity of $O(d_{model}(V+S)+6Ld_{model}^{2}+2Ld_{model})$, again shrinking nearly half of the cost. If PEFT method is used, take Lora as an example, let $t$ be the interval between two alignments during FL fine-tuning, $q$ be the proportion of neurons that need to be updated during alignment, $r$ be the rank of Lora, and Lora is only added to MHA, then the communication cost will be in the complexity of $O(8Lrd_{model}+\frac{2}{t}qLd_{model}^{2})$. Since $\frac{1}{t}qd_{model}$ is usually smaller than $r$, the complexity is about $O(Lrd_{model})$, which does not impose a significant computational cost.

\section{Experiments}

\subsection{Experimental Setting}
\subsubsection{Models and Datasets}
Our evaluations span a variety of FMs, including two NLP FMs (i.e., BERT-base \cite{kenton2019bert}, RoBERTa-base \cite{liu2019roberta}) and one CV FM (i.e., ViT-base \cite{dosovitskiy2020image}). Three FMs share consistent hyper-parameters for the number of layers $L$, attention heads $h$, hidden dimension $d_{model}$, and FFN dimension $d_{ff}$, all set at $L=12$, $h=12$, $d_{model}=768$, and $d_{ff}=3072$. Our NLP FM evaluations encompass four text datasets: SST-2 \cite{socher2013recursive}, QNLI \cite{socher2013recursive}, MNLI \cite{williams2017broad}, and QQP\footnote{https://quoradata.quora.com/First-Quora-Dataset-Release-Question-Pairs}. The CV FM is evaluated on three image datasets: CIFAR-10 \cite{krizhevsky2009learning}, CIFAR-100 \cite{krizhevsky2009learning} and Flowers \cite{nilsback2008automated}. Specifically, SST-2 is the text sentiment analysis task, QNLI and MNLI are the natural language inference tasks, QQP is the text matching task, while CIFAR-10, CIFAR-100 and Flowers focus on image recognition tasks. In addition, we employ the Bookcorpus \cite{zhu2015aligning} and Wikipedia datasets for distillation of NLP sub-FMs, and the ImageNet-1k \cite{russakovsky2015imagenet} for the distillation of CV sub-FM, respectively. Detailed dataset descriptions can be found in Appendix.C.
\subsubsection{Baselines}
We compare FedPFT with two methods, including: (1) FedPETuning \cite{zhang-etal-2023-fedpetuning} that performs parameter-efficient fine-tuning (PEFT) with full model through FL; and (2) FedOT that performs PEFT without full model through FL. It is worth noting that FedOT is the FL implementation of Offsite-Tuning \cite{xiao2023offsite} with multiple clients.
\subsubsection{Hyper-parameters and Implementation}
In the FL scenario, we set up 100 clients with 500 total communication rounds and employ the Dirichlet data partition method \cite{hsu2019measuring} to construct different label-skew data heterogeneity scenarios. In each communication round, we randomly select 10 clients for local fine-tuning, using a linear decay of the global learning rate over rounds and AdamW as the local fine-tuning optimizer. We use FedAvg \cite{mcmahan2017communication} for global model aggregation. For FedOT, following \cite{xiao2023offsite}, we use 2 layers at the bottom and 2 layers at the top as $Adapter$, and compress the intermediate 8 layers into 3 layers as $Emulator$. For FedPFT, we construct sub-FMs by performing layer-wise compression on the intermediate 10 layers. For fair comparison, we keep the number of trainable parameters the same for all three methods. The evaluation metric is the accuracy on the given validation set. All pre-trained models used in experiments are obtained from Hugging Face\footnote{https://huggingface.co/}. The FL scenarios were implemented using FLGo \cite{wang2023flgo}, and all experiments are conducted in PyTorch 2.1 and NVIDIA 3090 GPUs. Other detailed hyper-parameters can be seen in Appendix.C. 
\subsection{Overall Comparisons}
\begin{table}
    \centering
    \begin{tabular}{cccc}
    \hline
         Model& Method&CIFAR-10 & \makecell[c]{Transformers\\Params}\\
    \hline
         \multirow{3}*{ViT}&FedPETuning & 98.2& 81M\\
         \cline{2-4}
         & FedOT&95.5 & 47M\\
         & FedPFT&\textbf{97.2} & 47M\\
        \hline
    \end{tabular}
    \caption{\textbf{Experimental results on ViT}. The evaluation metric is accuracy.}
    \label{tab2}
\end{table}
\begin{table*}
    \centering
    \begin{tabular}{ccccccccc}
        \hline
         \multirow{2}*{Model} & \multirow{2}*{Method} & \multicolumn{3}{c}{SST-2}& \multicolumn{3}{c}{QNLI}&\multirow{2}*{\makecell[c]{Transformers\\Params}} \\
         \cline{3-8}
          &  &Dir-1.0& Dir-5.0 & Dir-10.0 &Dir-1.0& Dir-5.0 & Dir-10.0  &  \\
        \hline
        \multirow{3}*{BERT} & FedPETuning & 90.6 & 90.9 & 91.1 & 87.3 &87.9 & 88.1& 81M\\
         \cline{2-9}
         & FedOT & 89.1 & 89.9 & 90.1 & 82.4 & 82.9& 83.1& 47M \\
         
         & FedPFT&  \textbf{89.7}& \textbf{91.1} & \textbf{91.6} & \textbf{86.1} &\textbf{86.9}&\textbf{87.2} & 47M\\
        \hline
        \multirow{3}*{RoBERTa} & FedPETuning& 93.0 & 93.3 &93.5  &90.2  &90.6 &90.8 &81M\\
        \cline{2-9}
         & FedOT &  91.7& 92.1 & 92.8 & 84.1 & 84.6& 85.1&47M  \\
         & FedPFT & \textbf{92.1} & \textbf{92.7} & \textbf{93.1} &\textbf{87.2}  &\textbf{87.6} &\textbf{88.1} & 47M\\ 
         
        \hline
    \end{tabular}
    \caption{\textbf{Non-IID experimental results on BERT and RoBERTa}. The evaluation metric is accuracy.}
    \label{tab3}
\end{table*}
We first evaluate all three methods in Independent Identically Distributed (I.I.D) data distribution scenarios. Table.\ref{tab1} compares our FedPFT with two baselines for fine-tuning BERT and RoBERTa on four text datasets. We observe that: 1) the performance of all FMs fine-tuned by our FedPFT surpasses that achieved by FedOT and closely approaches the performance achieved by FedPETuning that fine-tuning using full model; and 2) FedOT exhibit a substantial performance gap compared to FedPETuning. This observed discrepancy from FedOT might be attributed to the absence of training the $Emulator$ and the accumulation of gradient errors during fine-tuning.\par
We further validate the effectiveness of FedPFT for fine-tuning CV FM. Table.\ref{tab2} presents the comparison results on the ViT model with CIFAR-10 dataset, and FedPFT still outperforms FedOT and achieves competitive performance  closer to FedPETuning. The results on other two vision datasets (i.e., CIFAR-100, Flowers) are shown in Appendix.D.
\subsection{Impact of Data Heterogeneity}
We then evaluate the performance of FMs fine-tuned by three different methods in Non-Independent Identically Distributed (Non-I.I.D) data distribution scenarios. For the datasets used in the data heterogeneity experiments, we unequally partition the dataset into 100 clients following the label distribution $Y_{i}\backsim Dir(\alpha p)$, where $i$ is the client id, $p$ is the global label distribution, $\alpha$ is the degree of Non-I.I.D and a smaller $\alpha$ generates a high label distribution shift. We construct three different label-skewing scenarios by adjusting the value of $\alpha$: Dir-1.0, Dir-5.0, and Dir-10.0. Table.\ref{tab3} shows the comparison of three methods for fine-tuning BERT and RoBERTa in different data-heterogeneous scenarios. We observe that: 1) the performance of all methods declines as the degree of Non-I.I.D increases; 2) our FedPFT still outperforms FedOT and achieves competitive performance closer to FedPETuning. Visualisation of label distributions and results of Non-I.I.D experiments on vision datasets are shown in the Appendix.E.
\subsection{Ablation Study}
\begin{table}[]
    \centering
    \begin{tabular}{ccc}
    \hline
         \multirow{2}*{Method}& \multicolumn{2}{c}{Model} \\
    \cline{2-3}
         &BERT&RoBERTa \\
    \hline
         FedOT&83.9 & 85.3 \\
         FedPFT\_N& 83.0 & 78.9\\
         FedPFT\_B&86.9 & 86.7 \\
         FedPFT\_D&84.3 & 85.5 \\
         FedPFT(ours)& \textbf{87.5}& \textbf{88.7}\\
    \hline
    \end{tabular}
    \caption{\textbf{Ablation study of FedPFT} on QNLI}
    \label{tab4}
\end{table}
To evaluate the efficacy of individual components within FedPFT, we design the following variants for conducting ablation study:
\begin{itemize}
    \item FedPFT\_N which does not perform alignment by knowledge distillation;
    \item FedPFT\_B which perform sub-FM alignment by knowledge distillation only before FL fine-tuning;
    \item FedPFT\_D which perform sub-FM alignment by knowledge distillation only during FL fine-tuning;
\end{itemize}
Moreover, to validate the effectiveness of the sub-FM construction module of FedPFT, we list the experimental results of FedOT for comparison with FedPFT\_B, as both perform sub-FM alignment by knowledge distillation only before FL fine-tuning. Results are shown in Table.\ref{tab4} and we observe that: 1) both FedPFT\_N which does not align the sub-FMs entirely and FedPFT\_D which lacks the alignment before FL fine-tuning, exhibit notably poor performance. This is consistent with our theoretical analysis (see Theorem.\ref{theorem_1}), indicating that a significant disparity between the gradients of sub-FM and FM can impede the convergence of fine-tuning methods using proxy sub-model; 2) FedPFT\_B outperforms FedOT, showing the effectiveness of the sub-FM construction module in FedPFT; and 3) FedPFT outperforms FedPFT\_B, emphasizing the necessity of sub-FM alignment during FL fine-tuning. Results on other text and vision datasets are presented in Appendix.F.
\subsection{Parameter Study}
\begin{table}[]
    \centering
    \begin{tabular}{cccc}
    \hline
         \multicolumn{2}{c}{\multirow{2}*{Hyper-Parameter}}& \multicolumn{2}{c}{Model} \\
    \cline{3-4}
          & & BERT & RoBERTa \\
    \hline
        \multirow{3}*{\makecell[c]{Alignment\\interval $t$}} &5 & 87.3& 88.4\\
         &10 & \textbf{87.5}& \textbf{88.7} \\
         &20 &87.1 & 87.1 \\
    \hline
        \multirow{3}*{\makecell[c]{Updated neurons\\proportion $p$}} &0.3 & 87.4& 87.8\\
         &0.5 &\textbf{87.5} & \textbf{88.7} \\
         &0.8 &87.1 & 87.9 \\
    \hline
    \end{tabular}
    \caption{\textbf{Parameter study of FedPFT} on QNLI}
    \label{tab5}
\end{table}
In addition, we conduct a study to investigate the impact of two hyper-parameters in the sub-FM alignment module: the interval $t$ between two alignments during FL fine-tuning and the proportion $p$ of neurons that need to be updated for each alignment. The effects of two hyper-parameters on the QNLI dataset are presented in Table.\ref{tab5}, suggesting that both $t$ and $p$ should be chosen moderately. Regarding $t$, longer alignment intervals lead to the accumulation of gradient errors, while shorter intervals can prohibit the adaptation of downstream tasks. Concerning $p$, it is crucial to strike a balance between updating an adequate proportion of neurons and avoiding excessive disruption to local fine-tuning.
\section{Conclusion}
This paper introduces FedPFT, a federated fine-tuning framework designed for Foundation Models (FMs). FedPFT addresses critical challenges related to insufficient FM fine-tuning and the accumulation of gradient errors by employing layer-wise compression for sub-FM construction and aligning sub-FM through a two-step distillation process, respectively. This novel framework achieves optimal downstream task adaptation of FM, resulting in an effective fine-tuned FM with superior performance, all without direct sharing of either server FM or client data. Experimental results across seven datasets showcase the effectiveness of FedPFT. In the future, we aim to extend the application of FedPFT to larger-scale FMs for tackling more complex downstream tasks.

\section{Acknowledgment}
The research was supported by Natural Science Foundation of China (62272403).
\bibliographystyle{named}
\bibliography{reference}

\begin{thebibliography}{}

\bibitem[\protect\citeauthoryear{Birhane and Prabhu}{2021}]{birhane2021large}
Abeba Birhane and Vinay~Uday Prabhu.
\newblock Large image datasets: A pyrrhic win for computer vision?
\newblock In {\em 2021 IEEE Winter Conference on Applications of Computer Vision (WACV)}, pages 1536--1546. IEEE, 2021.

\bibitem[\protect\citeauthoryear{Bommasani \bgroup \em et al.\egroup }{2021}]{bommasani2021opportunities}
Rishi Bommasani, Drew~A Hudson, Ehsan Adeli, Russ Altman, Simran Arora, Sydney von Arx, Michael~S Bernstein, Jeannette Bohg, Antoine Bosselut, Emma Brunskill, et~al.
\newblock On the opportunities and risks of foundation models.
\newblock {\em arXiv preprint arXiv:2108.07258}, 2021.

\bibitem[\protect\citeauthoryear{Chen \bgroup \em et al.\egroup }{2023a}]{chen2023federated}
Chaochao Chen, Xiaohua Feng, Jun Zhou, Jianwei Yin, and Xiaolin Zheng.
\newblock Federated large language model: A position paper.
\newblock {\em arXiv preprint arXiv:2307.08925}, 2023.

\bibitem[\protect\citeauthoryear{Chen \bgroup \em et al.\egroup }{2023b}]{chen2023fedbone}
Yiqiang Chen, Teng Zhang, Xinlong Jiang, Qian Chen, Chenlong Gao, and Wuliang Huang.
\newblock Fedbone: Towards large-scale federated multi-task learning.
\newblock {\em arXiv preprint arXiv:2306.17465}, 2023.

\bibitem[\protect\citeauthoryear{Chen \bgroup \em et al.\egroup }{2023c}]{ijcai2023p394}
Yuanyuan Chen, Zichen Chen, Pengcheng Wu, and Han Yu.
\newblock Fedobd: Opportunistic block dropout for efficiently training large-scale neural networks through federated learning.
\newblock In Edith Elkind, editor, {\em Proceedings of the Thirty-Second International Joint Conference on Artificial Intelligence, {IJCAI-23}}, pages 3541--3549. International Joint Conferences on Artificial Intelligence Organization, 8 2023.
\newblock Main Track.

\bibitem[\protect\citeauthoryear{Dosovitskiy \bgroup \em et al.\egroup }{2020}]{dosovitskiy2020image}
Alexey Dosovitskiy, Lucas Beyer, Alexander Kolesnikov, Dirk Weissenborn, Xiaohua Zhai, Thomas Unterthiner, Mostafa Dehghani, Matthias Minderer, Georg Heigold, Sylvain Gelly, et~al.
\newblock An image is worth 16x16 words: Transformers for image recognition at scale.
\newblock In {\em International Conference on Learning Representations}, 2020.

\bibitem[\protect\citeauthoryear{Fan \bgroup \em et al.\egroup }{2023}]{fan2023fate}
Tao Fan, Yan Kang, Guoqiang Ma, Weijing Chen, Wenbin Wei, Lixin Fan, and Qiang Yang.
\newblock Fate-llm: A industrial grade federated learning framework for large language models.
\newblock {\em arXiv preprint arXiv:2310.10049}, 2023.

\bibitem[\protect\citeauthoryear{Hinton \bgroup \em et al.\egroup }{2015}]{hinton2015distilling}
Geoffrey Hinton, Oriol Vinyals, and Jeff Dean.
\newblock Distilling the knowledge in a neural network.
\newblock {\em arXiv preprint arXiv:1503.02531}, 2015.

\bibitem[\protect\citeauthoryear{Houlsby \bgroup \em et al.\egroup }{2019}]{houlsby2019parameter}
Neil Houlsby, Andrei Giurgiu, Stanislaw Jastrzebski, Bruna Morrone, Quentin De~Laroussilhe, Andrea Gesmundo, Mona Attariyan, and Sylvain Gelly.
\newblock Parameter-efficient transfer learning for nlp.
\newblock In {\em International Conference on Machine Learning}, pages 2790--2799. PMLR, 2019.

\bibitem[\protect\citeauthoryear{Hsu \bgroup \em et al.\egroup }{2019}]{hsu2019measuring}
Tzu-Ming~Harry Hsu, Hang Qi, and Matthew Brown.
\newblock Measuring the effects of non-identical data distribution for federated visual classification.
\newblock {\em arXiv preprint arXiv:1909.06335}, 2019.

\bibitem[\protect\citeauthoryear{Hu \bgroup \em et al.\egroup }{2016}]{hu2016network}
Hengyuan Hu, Rui Peng, Yu-Wing Tai, and Chi-Keung Tang.
\newblock Network trimming: A data-driven neuron pruning approach towards efficient deep architectures.
\newblock {\em arXiv preprint arXiv:1607.03250}, 2016.

\bibitem[\protect\citeauthoryear{Hu \bgroup \em et al.\egroup }{2021}]{hu2021lora}
Edward~J Hu, Phillip Wallis, Zeyuan Allen-Zhu, Yuanzhi Li, Shean Wang, Lu~Wang, Weizhu Chen, et~al.
\newblock Lora: Low-rank adaptation of large language models.
\newblock In {\em International Conference on Learning Representations}, 2021.

\bibitem[\protect\citeauthoryear{Kenton and Toutanova}{2019}]{kenton2019bert}
Jacob Devlin Ming-Wei~Chang Kenton and Lee~Kristina Toutanova.
\newblock Bert: Pre-training of deep bidirectional transformers for language understanding.
\newblock In {\em Proceedings of NAACL-HLT}, pages 4171--4186, 2019.

\bibitem[\protect\citeauthoryear{Krizhevsky \bgroup \em et al.\egroup }{2009}]{krizhevsky2009learning}
Alex Krizhevsky, Geoffrey Hinton, et~al.
\newblock Learning multiple layers of features from tiny images.
\newblock 2009.

\bibitem[\protect\citeauthoryear{Kuang \bgroup \em et al.\egroup }{2023}]{kuang2023federatedscope}
Weirui Kuang, Bingchen Qian, Zitao Li, Daoyuan Chen, Dawei Gao, Xuchen Pan, Yuexiang Xie, Yaliang Li, Bolin Ding, and Jingren Zhou.
\newblock Federatedscope-llm: A comprehensive package for fine-tuning large language models in federated learning.
\newblock {\em arXiv preprint arXiv:2309.00363}, 2023.

\bibitem[\protect\citeauthoryear{Li and Liang}{2021}]{li2021prefix}
Xiang~Lisa Li and Percy Liang.
\newblock Prefix-tuning: Optimizing continuous prompts for generation.
\newblock In {\em Proceedings of the 59th Annual Meeting of the Association for Computational Linguistics and the 11th International Joint Conference on Natural Language Processing (Volume 1: Long Papers)}, pages 4582--4597, 2021.

\bibitem[\protect\citeauthoryear{Liu \bgroup \em et al.\egroup }{2019}]{liu2019roberta}
Yinhan Liu, Myle Ott, Naman Goyal, Jingfei Du, Mandar Joshi, Danqi Chen, Omer Levy, Mike Lewis, Luke Zettlemoyer, and Veselin Stoyanov.
\newblock Roberta: A robustly optimized bert pretraining approach.
\newblock {\em arXiv preprint arXiv:1907.11692}, 2019.

\bibitem[\protect\citeauthoryear{Mallya and Lazebnik}{2018}]{mallya2018packnet}
Arun Mallya and Svetlana Lazebnik.
\newblock Packnet: Adding multiple tasks to a single network by iterative pruning.
\newblock In {\em Proceedings of the IEEE conference on Computer Vision and Pattern Recognition}, pages 7765--7773, 2018.

\bibitem[\protect\citeauthoryear{Marchisio \bgroup \em et al.\egroup }{2023}]{marchisio-etal-2023-mini}
Kelly Marchisio, Patrick Lewis, Yihong Chen, and Mikel Artetxe.
\newblock Mini-model adaptation: Efficiently extending pretrained models to new languages via aligned shallow training.
\newblock In Anna Rogers, Jordan Boyd-Graber, and Naoaki Okazaki, editors, {\em Findings of the Association for Computational Linguistics: ACL 2023}, pages 5474--5490, Toronto, Canada, July 2023. Association for Computational Linguistics.

\bibitem[\protect\citeauthoryear{McMahan \bgroup \em et al.\egroup }{2017}]{mcmahan2017communication}
Brendan McMahan, Eider Moore, Daniel Ramage, Seth Hampson, and Blaise~Aguera y~Arcas.
\newblock Communication-efficient learning of deep networks from decentralized data.
\newblock In {\em Artificial intelligence and statistics}, pages 1273--1282. PMLR, 2017.

\bibitem[\protect\citeauthoryear{Nilsback and Zisserman}{2008}]{nilsback2008automated}
Maria-Elena Nilsback and Andrew Zisserman.
\newblock Automated flower classification over a large number of classes.
\newblock In {\em 2008 Sixth Indian conference on computer vision, graphics \& image processing}, pages 722--729. IEEE, 2008.

\bibitem[\protect\citeauthoryear{Radford \bgroup \em et al.\egroup }{}]{radford2018improving}
Alec Radford, Karthik Narasimhan, Tim Salimans, Ilya Sutskever, et~al.
\newblock Improving language understanding by generative pre-training.

\bibitem[\protect\citeauthoryear{Russakovsky \bgroup \em et al.\egroup }{2015}]{russakovsky2015imagenet}
Olga Russakovsky, Jia Deng, Hao Su, Jonathan Krause, Sanjeev Satheesh, Sean Ma, Zhiheng Huang, Andrej Karpathy, Aditya Khosla, Michael Bernstein, et~al.
\newblock Imagenet large scale visual recognition challenge.
\newblock {\em International journal of computer vision}, 115:211--252, 2015.

\bibitem[\protect\citeauthoryear{Sajjad \bgroup \em et al.\egroup }{2023}]{sajjad2023effect}
Hassan Sajjad, Fahim Dalvi, Nadir Durrani, and Preslav Nakov.
\newblock On the effect of dropping layers of pre-trained transformer models.
\newblock {\em Computer Speech \& Language}, 77:101429, 2023.

\bibitem[\protect\citeauthoryear{Socher \bgroup \em et al.\egroup }{2013}]{socher2013recursive}
Richard Socher, Alex Perelygin, Jean Wu, Jason Chuang, Christopher~D Manning, Andrew~Y Ng, and Christopher Potts.
\newblock Recursive deep models for semantic compositionality over a sentiment treebank.
\newblock In {\em Proceedings of the 2013 conference on empirical methods in natural language processing}, pages 1631--1642, 2013.

\bibitem[\protect\citeauthoryear{Sung \bgroup \em et al.\egroup }{2022}]{sung2022lst}
Yi-Lin Sung, Jaemin Cho, and Mohit Bansal.
\newblock Lst: Ladder side-tuning for parameter and memory efficient transfer learning.
\newblock {\em Advances in Neural Information Processing Systems}, 35:12991--13005, 2022.

\bibitem[\protect\citeauthoryear{Touvron \bgroup \em et al.\egroup }{2023}]{touvron2023llama}
Hugo Touvron, Thibaut Lavril, Gautier Izacard, Xavier Martinet, Marie-Anne Lachaux, Timoth{\'e}e Lacroix, Baptiste Rozi{\`e}re, Naman Goyal, Eric Hambro, Faisal Azhar, et~al.
\newblock Llama: Open and efficient foundation language models.
\newblock {\em arXiv preprint arXiv:2302.13971}, 2023.

\bibitem[\protect\citeauthoryear{Vaswani \bgroup \em et al.\egroup }{2017}]{vaswani2017attention}
Ashish Vaswani, Noam Shazeer, Niki Parmar, Jakob Uszkoreit, Llion Jones, Aidan~N Gomez, {\L}ukasz Kaiser, and Illia Polosukhin.
\newblock Attention is all you need.
\newblock {\em Advances in neural information processing systems}, 30, 2017.

\bibitem[\protect\citeauthoryear{Wang \bgroup \em et al.\egroup }{2018}]{wang2018glue}
Alex Wang, Amanpreet Singh, Julian Michael, Felix Hill, Omer Levy, and Samuel~R Bowman.
\newblock Glue: A multi-task benchmark and analysis platform for natural language understanding.
\newblock {\em arXiv preprint arXiv:1804.07461}, 2018.

\bibitem[\protect\citeauthoryear{Wang \bgroup \em et al.\egroup }{2023a}]{ijcai2023p483}
Xin'ao Wang, Huan Li, Ke~Chen, and Lidan Shou.
\newblock Fedbfpt: An efficient federated learning framework for bert further pre-training.
\newblock In Edith Elkind, editor, {\em Proceedings of the Thirty-Second International Joint Conference on Artificial Intelligence, {IJCAI-23}}, pages 4344--4352. International Joint Conferences on Artificial Intelligence Organization, 8 2023.
\newblock Main Track.

\bibitem[\protect\citeauthoryear{Wang \bgroup \em et al.\egroup }{2023b}]{wang2023flgo}
Zheng Wang, Xiaoliang Fan, Zhaopeng Peng, Xueheng Li, Ziqi Yang, Mingkuan Feng, Zhicheng Yang, Xiao Liu, and Cheng Wang.
\newblock Flgo: A fully customizable federated learning platform.
\newblock {\em arXiv preprint arXiv:2306.12079}, 2023.

\bibitem[\protect\citeauthoryear{Wen \bgroup \em et al.\egroup }{2016}]{wen2016learning}
Wei Wen, Chunpeng Wu, Yandan Wang, Yiran Chen, and Hai Li.
\newblock Learning structured sparsity in deep neural networks.
\newblock {\em Advances in neural information processing systems}, 29, 2016.

\bibitem[\protect\citeauthoryear{Williams \bgroup \em et al.\egroup }{2017}]{williams2017broad}
Adina Williams, Nikita Nangia, and Samuel~R Bowman.
\newblock A broad-coverage challenge corpus for sentence understanding through inference.
\newblock {\em arXiv preprint arXiv:1704.05426}, 2017.

\bibitem[\protect\citeauthoryear{Xiao \bgroup \em et al.\egroup }{2023}]{xiao2023offsite}
Guangxuan Xiao, Ji~Lin, and Song Han.
\newblock Offsite-tuning: Transfer learning without full model.
\newblock {\em arXiv preprint arXiv:2302.04870}, 2023.

\bibitem[\protect\citeauthoryear{Xu \bgroup \em et al.\egroup }{2023}]{xu2023federated}
Mengwei Xu, Yaozong Wu, Dongqi Cai, Xiang Li, and Shangguang Wang.
\newblock Federated fine-tuning of billion-sized language models across mobile devices.
\newblock {\em arXiv preprint arXiv:2308.13894}, 2023.

\bibitem[\protect\citeauthoryear{Yu \bgroup \em et al.\egroup }{2023}]{yu2023federated}
Sixing Yu, J~Pablo Mu{\~n}oz, and Ali Jannesari.
\newblock Federated foundation models: Privacy-preserving and collaborative learning for large models.
\newblock {\em arXiv preprint arXiv:2305.11414}, 2023.

\bibitem[\protect\citeauthoryear{Zhang \bgroup \em et al.\egroup }{2023}]{zhang-etal-2023-fedpetuning}
Zhuo Zhang, Yuanhang Yang, Yong Dai, Qifan Wang, Yue Yu, Lizhen Qu, and Zenglin Xu.
\newblock {F}ed{PET}uning: When federated learning meets the parameter-efficient tuning methods of pre-trained language models.
\newblock In Anna Rogers, Jordan Boyd-Graber, and Naoaki Okazaki, editors, {\em Findings of the Association for Computational Linguistics: ACL 2023}, pages 9963--9977, Toronto, Canada, July 2023. Association for Computational Linguistics.

\bibitem[\protect\citeauthoryear{Zhu \bgroup \em et al.\egroup }{2015}]{zhu2015aligning}
Yukun Zhu, Ryan Kiros, Rich Zemel, Ruslan Salakhutdinov, Raquel Urtasun, Antonio Torralba, and Sanja Fidler.
\newblock Aligning books and movies: Towards story-like visual explanations by watching movies and reading books.
\newblock In {\em Proceedings of the IEEE international conference on computer vision}, pages 19--27, 2015.

\bibitem[\protect\citeauthoryear{Zhuang \bgroup \em et al.\egroup }{2023}]{zhuang2023foundation}
Weiming Zhuang, Chen Chen, and Lingjuan Lyu.
\newblock When foundation model meets federated learning: Motivations, challenges, and future directions.
\newblock {\em arXiv preprint arXiv:2306.15546}, 2023.

\end{thebibliography}

\newpage
\appendix

\section{Proof of Theorem 1}
\begin{proof}
    Our assumption that $\triangledown f$ is Lipschitz continuous with constant $L_{1}$ implies that $\triangledown^{2} f(x) \preceq L_{1}I$, or equivalently that $\triangledown^{2} f(x)-L_{1}I$ is a negative semidefinite matrix. Using this fact, we can perform a quadratic expansion of $f$ around $f(x)$ and obtain the following inequality:
    \begin{equation}
        \begin{aligned}
        f(y) &\leq f(x) + \triangledown f(x)^{T}(y-x)+\frac{1}{2}\triangledown^{2}f(x)\|y-x\|_{2}^{2}\\
        &\leq f(x) + \triangledown f(x)^{T}(y-x)+\frac{L_{1}}{2}\|y-x\|_{2}^{2}
    \end{aligned}
    \end{equation}
    Since we run gradient descent on $f'$ and synchronize the gradient to $f$ and $\triangledown f' - \triangledown f=\delta$, let's plug in the gradient descent update by letting $y=x^{+}=x-\eta\triangledown f'(x)=x-\eta\triangledown f(x)-\eta\delta$. We then get:
    \begin{equation}
    \label{appendix_eq2}
    \begin{aligned}
            f(x^{+})&\leq f(x) + \triangledown f(x)^{T}(x^{+}-x)+\frac{L_{1}}{2}\|x^{+}-x\|_{2}^{2}\\
            f(x^{+})&\leq f(x)+\triangledown f(x)^{T}(-\eta\triangledown f(x)-\eta\delta)\\
            &+\frac{L_{1}}{2}\|\eta\triangledown f(x)+\eta\delta\|_{2}^{2}\\
            f(x^{+})&\leq f(x)-\eta(\|\triangledown f(x)\|_{2}^{2}+\triangledown f(x)^{T}\delta)\\
            &+\frac{L_{1}\eta^{2}}{2}\|\triangledown f(x)+\delta\|_{2}^{2}\\
            f(x^{+})&\leq f(x)-\eta(\|\triangledown f(x)\|_{2}^{2}+\triangledown f(x)^{T}\delta)\\
            &+\frac{L_{1}\eta^{2}}{2}(\|\triangledown f(x)\|_{2}^{2}+\|\delta\|_{2}^{2})\\
            f(x^{+})&\leq f(x)-(1-\frac{L_{1}\eta}{2})\eta\|\triangledown f(x)\|_{2}^{2}-\eta\triangledown f(x)^{T}\delta\\
            &+\frac{L_{1}\eta^{2}}{2}\|\delta\|_{2}^{2}
        \end{aligned}
    \end{equation}
    Using $\eta \leq \frac{1}{L_{1}}$, we know that $-(1-\frac{L_{1}\eta}{2})\leq-\frac{1}{2}$ and $\frac{L_{1}\eta^{2}}{2}\leq\frac{\eta}{2}$. Plugging them into (\ref{appendix_eq2}), we can get the following:
    \begin{equation}
    \label{appendix_eq3}
    \begin{aligned}
            f(x^{+})&\leq f(x) - \frac{\eta}{2}\|\triangledown f(x)\|_{2}^{2}-\eta\triangledown f(x)^{T}\delta+\frac{\eta}{2}\|\delta\|_{2}^{2}\\
            &=f(x)-\frac{\eta}{2}(\|\triangledown f(x)+\delta\|_{2}^{2}-2\|\delta\|_{2}^{2})\\
            &=f(x)-\frac{\eta}{2}(\|\triangledown f'(x)\|_{2}^{2}-2\|\delta\|_{2}^{2})
        \end{aligned}
    \end{equation}
    If we have 
    \begin{align}
    \label{appendix_eq4}
        \|\delta\|_{2}^{2}& \textless \frac{1}{2}\|\triangledown f'\|_{2}^{2},
    \end{align}
    then 
    \begin{align}
        f(x^{+})&\leq f(x).
    \end{align}
    This implies that when (\ref{appendix_eq4}) is satisfied, the value of $f$ strictly decreases with each iteration of gradient decent of $f'$ until it reaches the optimal value $f(x)=f(x^{*})$.\par
    Next, we can bound $f(x^{+})$, the objective value at the next iteration, in terms of $f(x^{*})$, the optimal objective value. Since $f$ is convex, we can write:
    \begin{align}
            &f(x^{*})\geq f(x)+\triangledown f(x)^{T}(x^{*}-x),\\
            &f(x)\leq f(x^{*})+\triangledown f(x)^{T}(x-x^{*}),
    \end{align}
    plugging them into (\ref{appendix_eq3}), we obtain:
    \begin{equation}
        \begin{aligned}
            f(x^{+})\leq f(x^{*})&+\triangledown f(x)^{T}(x-x^{*})\\&
            -\frac{\eta}{2}(\|\triangledown f'(x)\|_{2}^{2}-2\|\delta\|_{2}^{2})\\
            f(x^{+})-f(x^{*})&\leq (\triangledown f'(x)-\delta)^{T}(x-x^{*})\\
            &-\frac{\eta}{2}(\|\triangledown f'(x)\|_{2}^{2}-2\|\delta\|_{2}^{2})\\
            f(x^{+})-f(x^{*})&\leq \frac{1}{2\eta}(2\eta\triangledown f'(x)^{T}(x-x^{*})-\eta^{2}\|\triangledown f'(x)\|_{2}^{2})\\
            &-\delta^{T}(x-x^{*})+\eta\|\delta\|_{2}^{2}\\
            f(x^{+})-f(x^{*})&\leq\frac{1}{2\eta}(2\eta\triangledown f'(x)^{T}(x-x^{*})-\eta^{2}\|\triangledown f'(x)\|_{2}^{2}\\
            &-\|x-x^{*}\|_{2}^{2})+\frac{1}{2\eta}\|x-x^{*}\|_{2}^{2}\\
            &-\delta^{T}(x-x^{*})+\eta\|\delta\|_{2}^{2}\\
            f(x^{+})-f(x^{*})&\leq-\frac{1}{2\eta}\|x-x^{*}-\eta\triangledown f'(x)\|_{2}^{2}+\frac{1}{2\eta}\|x-x^{*}\|_{2}^{2}\\
            &-\delta^{T}(x-x^{*})+\eta\|\delta\|_{2}^{2}\\
            f(x^{+})-f(x^{*})&\leq \frac{1}{2\eta}(\|x-x^{*}\|_{2}^{2}-\|x^{+}-x^{*}\|_{2}^{2})\\
            &-\delta^{T}(x-x^{*})+\eta\|\delta\|_{2}^{2}\\
        \end{aligned}
    \end{equation}
    This inequality holds for $x^{+}$ on every update iteration. Summing over iterations, we get:
    \begin{equation}
        \begin{aligned}
                \sum_{i=1}^{k}f(x^{(i)})-f(x^{*})&\leq\sum_{i=1}^{k}\frac{1}{2\eta}(\|x^{(i-1)}-x^{*}\|_{2}^{2}-\|x^{(i)}-x^{*}\|_{2}^{2})\\
                &-\delta^{(i)T}(x^{(i)}-x^{*})+\eta\|\delta^{(i)}\|_{2}^{2}\\
                \sum_{i=1}^{k}f(x^{(i)})-f(x^{*})&\leq\frac{1}{2\eta}(\|x^{(0)}-x^{*}\|_{2}^{2}-\|x^{(k)}-x^{*}\|_{2}^{2})\\
                &+\sum_{i=1}^{k}\eta\|\delta^{(i)}\|_{2}^{2}-\delta^{(i)T}(x^{(i)}-x^{*})\\
                \sum_{i=1}^{k}f(x^{(i)})-f(x^{*})&\leq\frac{1}{2\eta}\|x^{(0)}-x^{*}\|_{2}^{2}\\
                &+\sum_{i=1}^{k}\eta\|\delta^{(i)}\|_{2}^{2}-\delta^{(i)T}(x^{(i)}-x^{*})\\
            \end{aligned}
    \end{equation}
    Then, using the fact that $f$ decreasing on every iteration, we can conclude that:
    \begin{equation}
        \begin{aligned}
            f(x^{(k)})-f(x^{*})&\leq\frac{1}{k}\sum_{i=1}^{k}f(x^{(i)})-f(x^{*})\\
            f(x^{(k)})-f(x^{*})&\leq\frac{1}{2\eta k}\|x^{(0)}-x^{*}\|_{2}^{2}\\
                &+\frac{1}{k}\sum_{i=1}^{k}\eta\|\delta^{(i)}\|_{2}^{2}-\delta^{(i)T}(x^{(i)}-x^{*})
        \end{aligned}
    \end{equation}
    If we have
    \begin{equation}
        \begin{aligned}
        \label{appendix_eq11}\eta\sum_{i=1}^{k}\|\delta^{(i)}\|_{2}^{2}&\leq \sum_{i=1}^{k}\langle \delta^{(i)}, x^{(i)}-x^{*} \rangle,
        \end{aligned}
    \end{equation}
    then
    \begin{equation}
        \begin{aligned}
            f(x^{(k)})-f(x^{*}) \leq \frac{\|x^{(0)}-x^{*}\|_{2}^{2}}{2\eta k}
        \end{aligned}
    \end{equation}
    This implies that when (\ref{appendix_eq11}) is satisfied, gradient descent of $f$ with the help of $f'$ is guaranteed to converge and converges with rate $O(\frac{1}{k})$.
\end{proof}


\section{Proof of Theorem 2}
\begin{proof}
    Based on the premise, we have:
    \begin{equation}
        \begin{aligned}
            loss&=f(y),\quad y=xAx^{T}xBC;\\
            loss'&=f(y'),\quad y'=xAx^{T}xBC'.
        \end{aligned}
    \end{equation}
    According to the chain rule, we can obtain:
    \begin{equation}
        \begin{aligned}
            \frac{\partial loss}{\partial A}&=x^{T}\frac{\partial loss}{\partial y}C^{T}B^{T}x^{T}x,\\
            \frac{\partial loss'}{\partial A}&=x^{T}\frac{\partial loss'}{\partial y'}C'^{T}B^{T}x^{T}x,
        \end{aligned}
    \end{equation}
    then
    \begin{equation}
        \begin{aligned}
        \label{appendix_eq16}
            \|\frac{\partial loss'}{\partial A}-\frac{\partial loss}{\partial A}\|_{2}^{2}&=\|x^{T}(\frac{\partial loss}{\partial y}C^{T}-\frac{\partial loss'}{\partial y'}C'^{T})B^{T}x^{T}x\|_{2}^{2}\\
            \|\frac{\partial loss'}{\partial A}-\frac{\partial loss}{\partial A}\|_{2}^{2}&\leq\|x\|_{2}^{6}\cdot\|B\|_{2}^{2}\cdot\|\frac{\partial loss}{\partial y}C^{T}-\frac{\partial loss'}{\partial y'}C'^{T}\|_{2}^{2}\\
            \|\frac{\partial loss'}{\partial A}-\frac{\partial loss}{\partial A}\|_{2}^{2}&\leq\|x\|_{2}^{6}\cdot\|B\|_{2}^{2}\cdot\|\frac{\partial loss}{\partial y}C^{T}-\frac{\partial loss'}{\partial y'}C^{T}\\
            &+\frac{\partial loss'}{\partial y'}C^{T}-\frac{\partial loss'}{\partial y'}C'^{T}\|_{2}^{2}\\
            \|\frac{\partial loss'}{\partial A}-\frac{\partial loss}{\partial A}\|_{2}^{2}&\leq\|x\|_{2}^{6}\cdot\|B\|_{2}^{2}\cdot\|C\|_{2}^{2}\cdot\|\frac{\partial loss}{\partial y}-\frac{\partial loss'}{\partial y'}\|_{2}^{2}\\
            &+\|x\|_{2}^{6}\cdot\|B\|_{2}^{2}\cdot\|\frac{\partial loss'}{\partial y'}\|_{2}^{2}\cdot\|C-C'\|_{2}^{2}.
        \end{aligned}
    \end{equation}
    Additionally, our assumption that the gradient of loss function $loss=f(y)$ is Lipschitz continuous with constant $L_{3} \textgreater 0$ implies that 
    \begin{equation}
        \begin{aligned}
            \|\frac{\partial loss}{\partial y}-\frac{\partial loss'}{\partial y'}\|_{2}^{2}\leq L_{3}\|y-y'\|_{2}^{2}.
        \end{aligned}
    \end{equation}
    Plugging this into (\ref{appendix_eq16}), we have
    \begin{equation}
        \begin{aligned}
            \|\frac{\partial loss'}{\partial A}-\frac{\partial loss}{\partial A}\|_{2}^{2}&\leq L_{3}\|x\|_{2}^{6}\cdot\|B\|_{2}^{2}\cdot\|C\|_{2}^{2}\cdot\|y-y'\|_{2}^{2}\\
            &+\|x\|_{2}^{6}\cdot\|B\|_{2}^{2}\cdot\|\frac{\partial loss'}{\partial y'}\|_{2}^{2}\cdot\|C-C'\|_{2}^{2}.
        \end{aligned}
    \end{equation}
    let 
    \begin{equation}
        \begin{aligned}
            K_{1}&=L_{3}\|x\|_{2}^{6}\cdot\|B\|_{2}^{2}\cdot\|C\|_{2}^{2}\\
            K_{2}&=\|x\|_{2}^{6}\cdot\|B\|_{2}^{2}\cdot\|\frac{\partial loss'}{\partial y'}\|_{2}^{2}
        \end{aligned}
    \end{equation}
    then we have
    \begin{equation}
        \begin{aligned}
            \|\frac{\partial loss'}{\partial A}-\frac{\partial loss}{\partial A}\|_{2}^{2} &\leq K_{1}\epsilon_{1} + K_{2}\epsilon_{2}
        \end{aligned}
    \end{equation}
    Similarly, it can be derived that there exists the constant $K_{3} \textgreater 0$ and $K_{4}\textgreater 0$ such that:
    \begin{equation}
        \begin{aligned}
            \|\frac{\partial loss'}{\partial B}-\frac{\partial loss}{\partial B}\|_{2}^{2} &\leq K_{3}\epsilon_{1} + K_{4}\epsilon_{2}.
        \end{aligned}
    \end{equation}
    Let $K_{1}=\max(K_{1}, K_{3})$, $K_{2}=\max(K_{2},K_{4})$, then Theorem 2 is proved.
\end{proof}

\section{Experimental Details}
\subsection{Datasets}
\subsubsection{SST-2} The Stanford Sentiment Treebank \cite{socher2013recursive} is a binary single-sentence classification task consisting of sentences extracted from movie reviews with human annotations of their sentiment. It is consists of a training set of 67350 examples, a development set of 873 examples, and a test set of 1821 examples.
\subsubsection{QNLI} Question Natural Language Inference is a version of the Stanford Question Answering Dataset which has been converted to a binary classification task \cite{wang2018glue}. It is consists of a training set of 104743 examples, a development set of 5463 examples, and a test set of 5461 examples.
\subsubsection{MNLI} Multi-Genre Natural Language Inference is a large-scale, crowdsourced entailment classification task \cite{williams2017broad}. It is consists of a training set of 392702 examples, a matched development set of 9815 examples, a mismatched development set of 9832 examples, a matched test set of 9796 examples, and a mismatched test set of 9847 examples. 
\subsubsection{QQP} Quora Question Pairs is a binary classification task where the goal is to determine if two questions asked on Quora are semantically equivalent \footnote{https://quoradata.quora.com/First-Quora-Dataset-Release-Question-Pairs}. It is consists of a training set of 363870 examples, a development set of 40431 examples, and a test set of 390965 examples.
\subsubsection{CIFAR-10 and CIFAR-100}The CIFAR-10 and CIFAR-100 are labeled subsets of the 80 million tiny images dataset \cite{birhane2021large}. The CIFAR-10 dataset consists of 60000 32x32 colour images in 10 classes, with 6000 images per class, and there are 50000 training images and 10000 test images. The CIFAR-100 dataset is just like the CIFAR-10, except it has 100 classes containing 600 images each, and there are 500 training images and 100 testing images per class. The 100 classes in the CIFAR-100 are grouped into 20 superclasses. Each image comes with a "fine" label (the class to which it belongs) and a "coarse" label (the superclass to which it belongs). 
\subsubsection{Flowers} Oxford 102 Flower is an image classification dataset consisting of 102 flower categories. The flowers were chosen to be flowers commonly occurring in the United Kingdom. Each class consists of between 40 and 258 images. It is consists of a training set of 6149 examples, a development set of 1020 examples, and a test set of 1020 examples.
\subsection{Hyper-parameters}

The number of distillation epochs for alignment before FL training for NLP sub-FMs and CV sub-FM is 5, 50, respectively. For distillation of NLP FMs, the learning rate is $6e-4$, the batch zise is 2048, and the weight decay is 0.01. For distillation of CV FMs, the learning rate is $1e-3$, the batch size is 4096, and the weight decay is 0.1. All distillation use the linear learning rate decay with warm up ratio of 0.06 and the gradient clipping with max grad norm of 1. For client local fine-tuning, the fine-tuning method is Lora \cite{hu2021lora}, the fine-tuning epoch is 1. For text datasets, we perform a grid search to find the optimal parameters, with batch size$\in \{8, 16, 32\}$ and learning rate$\in \{4e-4,8e-4,1e-3\}$. For CIFAR-10 and CIFAR-100 datasets, we perform a grid search to find the optimal parameters, with batch size$\in \{64, 128, 256\}$ and learning rate$\in \{3e-3,8e-3,1e-2\}$. For Flowers dataset, we perform a grid search to find the optimal parameters, with batch size$\in \{16, 32, 64\}$ and learning rate$\in \{3e-3,8e-3,1e-2\}$. For FedPFT, alignment is performed every 10 rounds during FL training, with the number of alignment epochs for NLP and CV being 0.02 and 0.2, respectively, and the proportion of neurons that need to be updated during alignment being 0.5. The ratio of eliminated neurons in FFN is fixed at 0.75.
\section{Experimental results on I.I.D scenario}
The I.I.D experiment results on CIFAR-100 and Flowers datasets are show in Table.\ref{tab6}. It is shown that FedPFT still outperforms FedOT and achieves competitive performance  closer to FedPETuning.

\section{Experimental results on Non-I.I.D scenario}
\begin{figure*}
\centering
    \subfigure[SST-2 Dir-1.0]{
    \centering
    \includegraphics[scale=0.333]{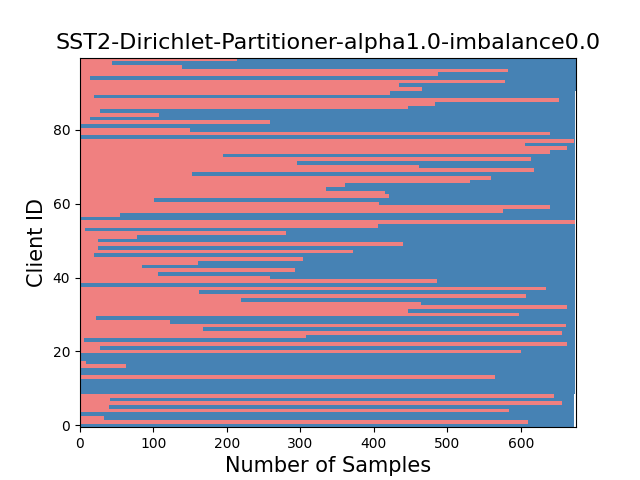}
    \label{fig4_a}
    }
    \subfigure[SST-2 Dir-5.0]{
    \centering
    \includegraphics[scale=0.333]{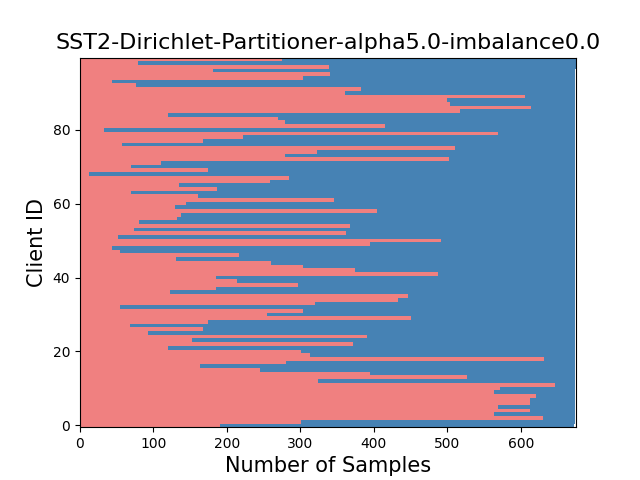}
    \label{fig4_b}
    }
    \subfigure[SST-2 Dir-10.0]{
    \centering
    \includegraphics[scale=0.333]{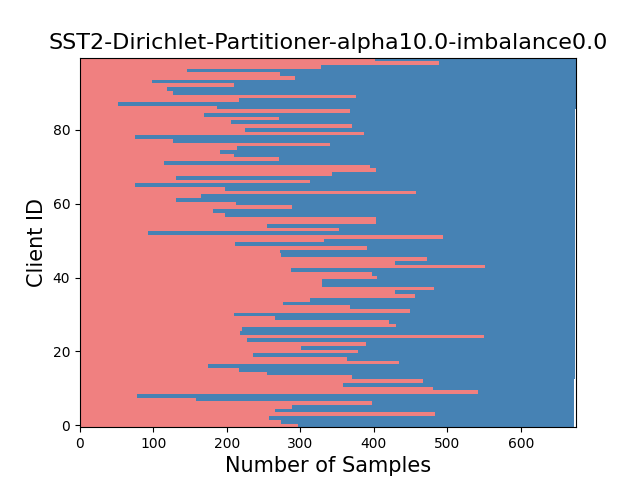}
    \label{fig4_c}
    }
    \subfigure[QNLI Dir-1.0]{
    \centering
    \includegraphics[scale=0.333]{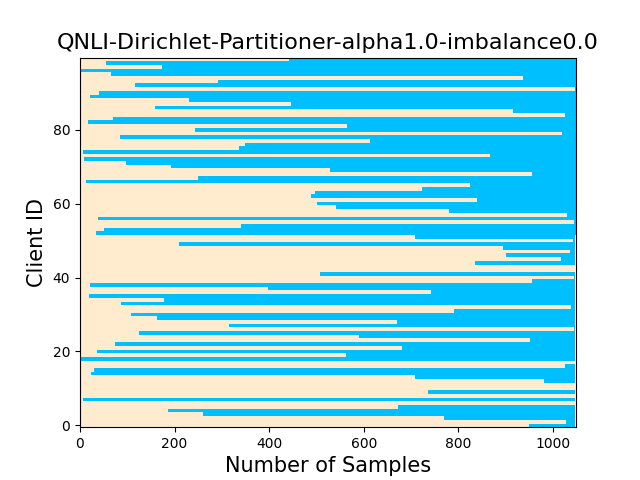}
    \label{fig5_a}
    }
    \subfigure[QNLI Dir-5.0]{
    \centering
    \includegraphics[scale=0.333]{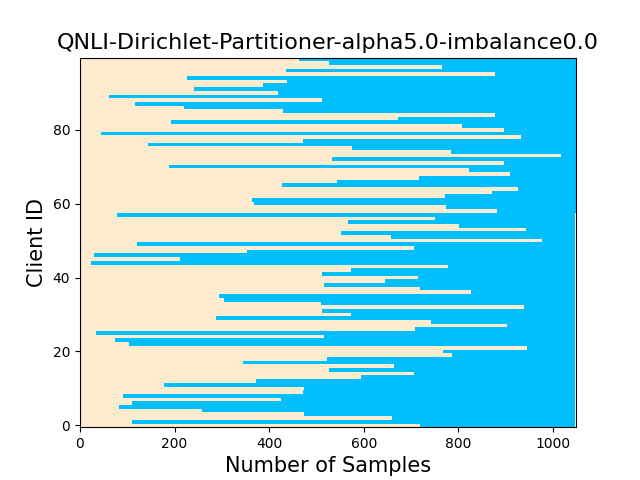}
    \label{fig5_b}
    }
    \subfigure[QNLI Dir-10.0]{
    \centering
    \includegraphics[scale=0.333]{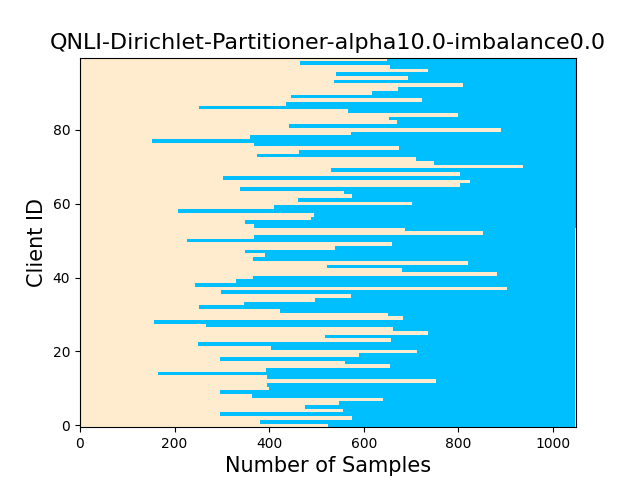}
    \label{fig5_c}
    }
    \subfigure[CIFAR-10 Dir-1.0]{
    \centering
    \includegraphics[scale=0.333]{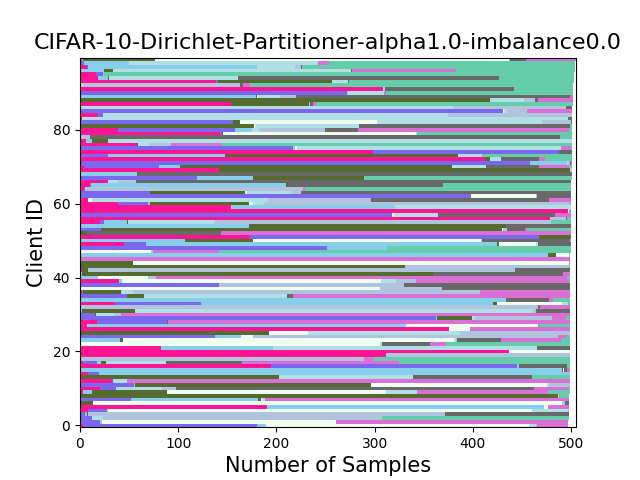}
    \label{fig6_a}
    }
    \subfigure[CIFAR-10 Dir-5.0]{
    \centering
    \includegraphics[scale=0.333]{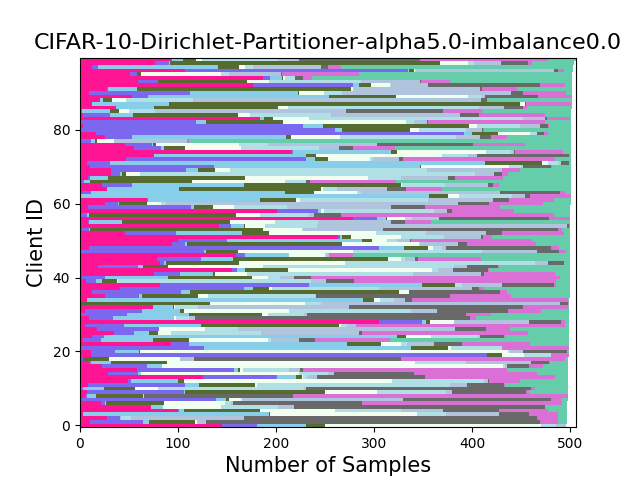}
    \label{fig6_b}
    }
    \subfigure[CIFAR-10 Dir-10.0]{
    \centering
    \includegraphics[scale=0.333]{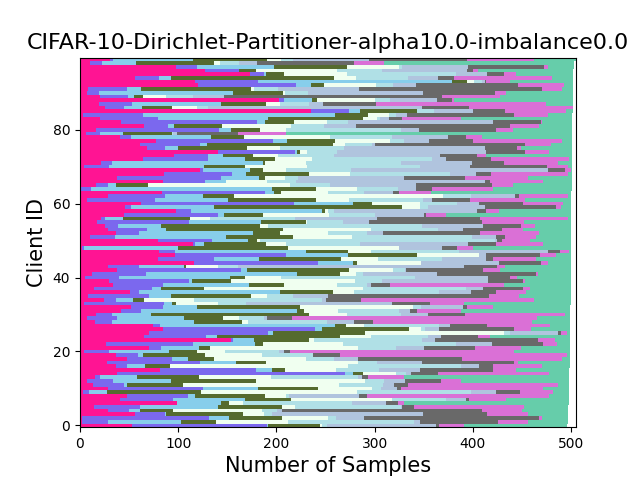}
    \label{fig6_c}
    }
    \caption{Visualisations of the label distributions of SST-2, QNLI, and CIFAR-10}
    \label{fig4}
\end{figure*}
Visualisations of the label distributions of the three datasets, SST-2, QNLI and CIFAR-10, under different Non-I.I.D scenarios are shown in Fig.\ref{fig4}. The results of Non-I.I.D experiments on CIFAR-10 are shown in Table.\ref{tab7}.  Similarly, We observe that: 1) the performance of all methods declines as the degree of Non-I.I.D increases; 2) our FedPFT still outperforms FedOT and achieves competitive performance closer to FedPETuning.
\begin{table}
    \centering
    \begin{tabular}{cccc}
    \hline
         Method&CIFAR-100&Flowers & \makecell[c]{Transformers\\Params}\\
    \hline
         FedPETuning &89.3 &99.7& 81M\\
         \cline{1-4}
         FedOT&81.1& 82.1& 47M\\
         FedPFT&\textbf{85.5}&\textbf{98.1} & 47M\\
        \hline
    \end{tabular}
    \caption{\textbf{Other I.I.D experimental results on ViT}. A higher value indicates better accuracy.}
    \label{tab6}
\end{table}
\begin{table}
    \centering
    \begin{tabular}{ccccc}
        \hline
          \multirow{2}*{Method} & \multicolumn{3}{c}{CIFAR-10}&\multirow{2}*{\makecell[c]{Transformers\\Params}} \\
         \cline{2-4}
           &Dir-1.0& Dir-5.0 & Dir-10.0 &  \\
        \hline
        FedPETuning & 97.8 & 98.0 &98.2 & 81M\\
         \cline{1-5}
         FedOT & 95.1 &95.4 & 95.5& 47M \\
         
         FedPFT&  \textbf{96.5} &\textbf{96.7}&\textbf{97.1} & 47M\\

        \hline
    \end{tabular}
    \caption{\textbf{Non-IID experimental results on CIFAR-10}. A higher value indicates better accuracy.}
    \label{tab7}
\end{table}
\section{Ablation study on other datasets}
The experiment results of ablation study on SST-2 and CIFAR-10 are shown in Table.\ref{tab8}. The experimental results exemplify the respective importance of the sub-FM Construction Module and the Sub-FM Alignment Module in FedPFT.
\begin{table}[]
    \centering
    \begin{tabular}{cccc}
    \hline
         Dataset&\multicolumn{2}{c}{SST-2}&CIFAR-10\\
    \hline
         Model&BERT&RoBERTa&ViT \\
    \hline
         FedOT&90.4 &92.8 &95.5 \\
         FedPFT\_N&84.2 & 89.1&95.1 \\
         FedPFT\_B& 90.8&92.9 &96.9  \\
         FedPFT\_D& 88.4& 89.7&96.0\\
         FedPFT(ours)& \textbf{91.6}& \textbf{93.1}&\textbf{97.2}\\
    \hline
    \end{tabular}
    \caption{\textbf{Ablation study of FedPFT} on SST-2 and CIFAR-10.}
    \label{tab8}
\end{table}

\section{Extension experimenal results}
We extended the experiment in three parts: parameter study on CIFAR-10 (shown in Table.\ref{tab9}), a new evaluation metric on QQP (shown in Table.\ref{tab10}), and convergence analysis of FedPFT (shown in Table.\ref{tab11}).
\begin{table}[H]
    \centering
    \begin{tabular}{ccccccc}
    \hline
         \multirow{2}*{Hyper-Params}&\multicolumn{3}{c}{$t$}& \multicolumn{3}{c}{$p$} \\
    \cline{2-7}
    &5&10 & 20& 0.1&0.5&0.9 \\
    \hline
        Accuracy&96.9&\textbf{97.2} & 97.1& 96.9&\textbf{97.2}&96.6 \\
    \hline
    \end{tabular}
    \caption{\textbf{Parameter study of FedPFT} on CIFAR-10.}
    \label{tab9}
\end{table}

\begin{table}[H]
    \centering
    \begin{tabular}{ccccc}
    \hline
         \multirow{2}*{Method}&\multicolumn{2}{c}{BERT}&\multicolumn{2}{c}{RoBERTa}\\
        \cline{2-5}
        &F1&rounds&F1&rounds \\
    \hline
         FedPETuning&83.5& 235 & 84.7&116 \\
        \hline
         FedOT& 77.9 &287& 77.7&\textbf{80}\\
         FedPFT& \textbf{81.5}&\textbf{274}& \textbf{82.6}&98\\
    \hline
    \end{tabular}
    \caption{\textbf{F1 score of three methods} on QQP. 'rounds' is total communication rounds of required for the model to converge.}
    \label{tab10}
\end{table}

\begin{table}[H]
    \centering
    \begin{tabular}{ccccc}
    \hline
         Method&SST-2&QNLI&QQP&MNLI \\
    \hline
         FedPETuning&18 & 6 &4 &11 \\
        \hline
         FedOT& 34 &128& 287&184\\
         FedPFT&\textbf{21} &\textbf{31}& \textbf{24}&\textbf{23}\\
    \hline
    \end{tabular}
    \caption{\textbf{Communication rounds required to achieve the same performance by all methods}. We list the communication rounds
required for all methods to achieve the same performance as FedOT.}
    \label{tab11}
\end{table}

\end{document}